\documentclass[twoside,10pt]{jmlr} 
\usepackage{times}
\usepackage{amsmath}
\usepackage{amssymb}
\usepackage{xspace}
\usepackage{bm}
\usepackage{algorithm}

\newtheorem{assumption}{Assumption}

\newlength{\minipagewidth}
\newcommand{\F}{\mathcal{F}}

\newcommand{\real}{\mathbb{R}}

\newcommand{\Sw}{\mathcal{S}}

\newcommand{\HH}{\mathcal{H}}
\newcommand{\TT}{\mathcal{T}}

\newcommand{\Iw}{\mathcal{I}}
\newcommand{\DD}{\mathcal{D}}
\newcommand{\UU}{\mathcal{U}}
\newcommand{\OO}{\mathcal{O}}

\newcommand{\trace}[1]{\mbox{tr}\left[#1\right]}
\newcommand{\II}[1]{\mathbb{I}_{\left\{#1\right\}}}
\newcommand{\PP}[1]{\mathbb{P}\left[#1\right]}
\newcommand{\E}{\mathbb{E}}
\newcommand{\EE}[1]{\mathbb{E}\left[#1\right]}

\newcommand{\EEt}[1]{\mathbb{E}_t\left[#1\right]}
\newcommand{\EEG}[1]{\mathbb{E}_G\left[#1\right]}
\newcommand{\EEk}[1]{\mathbb{E}_k\left[#1\right]}

\newcommand{\EEcc}[2]{\mathbb{E}\left[\left.#1\right|#2\right]}

\newcommand{\ra}{\rightarrow}

\newcommand{\norm}[1]{\left\|#1\right\|}

\newcommand{\ev}[1]{\left\{#1\right\}}
\newcommand{\pa}[1]{\left(#1\right)}
\newcommand{\bpa}[1]{\bigl(#1\bigr)}

\newcommand{\BPA}[1]{\Biggl(#1\Biggr)}

\newcommand{\wh}{\widehat}
\newcommand{\wt}{\widetilde}

\newcommand{\tw}{\wt{w}}
\newcommand{\bw}{\overline{w}}

\newcommand{\hw}{\wh{w}}

\newcommand{\iprod}[2]{\left\langle#1, #2\right\rangle}
\newcommand{\Iprod}[2]{\bigl\langle#1, #2\bigr\rangle}

\usepackage{todonotes}
\definecolor{PalePurp}{rgb}{0.66,0.57,0.66}

\title[Iterate averaging as regularization for SGD]{Iterate averaging as regularization for stochastic gradient 
descent} 

\author{\Name{Gergely Neu} \Email{gergely.neu@gmail.com}\\
 \addr Universitat Pompeu Fabra, Barcelona, Spain \\
 \AND
 \Name{Lorenzo Rosasco} \Email{lrosasco@mit.edu}\\
 \addr LCSL, Massachusetts Institute of Technology, Cambridge MA, USA
 \\
 Istituto Italiano di Tecnologia, Genova, Italy
 \\
 Universit\`a degli Studi di Genova, Genova, Italy
 }

\begin{document}

\maketitle

\begin{abstract}
 We propose and analyze  a variant of the classic  Polyak--Ruppert averaging scheme,  broadly used in stochastic gradient methods.  Rather 
than  a uniform average of the iterates, we consider a weighted average, with  weights decaying in a  geometric fashion.   In the context of 
linear least squares regression, we show that this averaging  scheme has a the same regularizing effect, and indeed is asymptotically equivalent,  to ridge regression. 
In particular, we derive finite-sample bounds for  the proposed approach that match  the best known results for regularized stochastic 
gradient methods.
\end{abstract}


\section{Introduction}

Stochastic gradient  methods are  ubiquitous in machine learning, where they are typically referred to as SGD (stochastic gradient 
descent\footnote{Albeit in general they might not  be descent methods.}). Since these incremental methods use little computation per data 
point, they are naturally adapted to processing very large data sets or streams of data.
Stochastic gradient methods have a long history, starting from the pioneering paper by Robbins and Monro \citep{RM51}. For a more 
detailed discussion, we refer to the excellent review given by \cite{NJLS09}. In the present paper, we propose a variant of SGD based on 
a weighted average of the iterates. The idea of averaging iterates goes back to \cite{P91} and \cite{R88}, and indeed it is 
often referred to as Polyak--Ruppert averaging (see also \citealp{PJ92}).  In this paper, we study SGD in the context of the linear 
least-squares regression problem, considering both  finite- and infinite-dimensional settings. This latter  case allows to derive results 
for nonparametric learning with kernel methods---we refer to the appendix in \cite{RV15} for a detailed discussion on this subject. 
The study of SGD for least squares is classical in stochastic approximation \citep{K03}, where it is commonly known as the 
\emph{least-mean-squares} (LMS) algorithm. 
In the context of machine learning theory, prediction-error guarantees for online algorithms can be derived through a regret analysis in a 
sequential-prediction setting and a so called online-to-batch conversion \citep{SS12,Haz16}. 
 For example, results  in \cite{Vov01,AW01} and \cite{HAK07} directly apply to  least squares. 
 Alternatively, one can directly analyze SGD in the stochastic setting, as done by \cite{YS06,YP08, TY14}, where the 
last iterate  and decaying step-size are considered, and more recently by 
  \cite{RV15,LR17} where  multiple passes and mini-batching are considered. 
A recently popular approach is combining \emph{constant} step-sizes with Polyak--Ruppert averaging, which was first shown to lead to  
strong finite-time prediction guarantees after a single pass on the data by \citet{BM13}. This approach was first studied by 
\citet{GyW96} and subsequent progress was made by \cite{DB15,DFB16,JKKNS16,JKKNPS17,LS18}. 

In this paper we propose and analyze a novel form of weighted average, given  by a sequence of weights  decaying geometrically, so that 
the first iterates have more weight. Our main technical contribution is a characterization of the properties of this particular 
weighting scheme that we call \emph{geometric Polyak--Ruppert averaging}. Our first 
result shows that SGD with geometric Polyak--Ruppert averaging is in expectation equivalent to considering SGD with a regularized loss 
function, and both sequences converge to the Tikhonov-regularized solution of the expected least-squares problem. The regularization 
parameter is a tuning parameter defining the sequence of geometric weights. This result strongly suggests that geometric Polyak--Ruppert 
averaging can be used to control the bias-variance properties of the corresponding SGD estimator. Indeed, our main result quantifies this 
intuition deriving a finite-sample bound,  matching previous results for regularized SGD, and   leading to   optimal rates \citep{T08}.
While averaging is widely considered to have a stabilizing effect,   to the best of 
our knowledge this is the first result  characterizing  the stability of an averaging scheme in terms of its regularization properties and 
corresponding prediction guarantees. Our findings can be contrasted to recent results on tail averaging \citep{JKKNS16} and provide some 
guidance on when and how different averaging strategies can be useful. On a high level, our results suggest that geometric averaging should 
be used when the data is poorly-conditioned and/or relatively small, and tail averaging should be used in the opposite case.
Further, from a practical point of view,  geometric Polyak--Ruppert averaging provides an efficient approach to perform model 
selection, since 
{\em a regularization path} \citep{HTF01} is computed efficiently.  Indeed, it is possible to compute a full pass of SGD  \emph{once} and 
store all the iterates, to then rapidly compute off-line the solutions corresponding to different geometric weights (or tail averages), 
hence different regularization levels. As the averaging operation is entirely non-serial, this method lends itself to trivially easy 
parallelization.

The rest of the paper is organized as follows. In Section~\ref{sec:pre}, we introduce the necessary background and present the geometric 
Polyak--Ruppert averaging scheme. In Section~\ref{sec:expectation}, we show the asymptotic equivalence between ridge regression and 
constant-stepsize SGD with geometric iterate averaging. Section~\ref{sec:main} presents and dicuss our main results regarding 
the finite-time prediction error of the method.
Section~\ref{sec:analysis} describes the main steps in the proofs. The full proof is included in the Appendix. 
We conclude this section by  introducing some basic notation used throughout the paper.

\paragraph{Notation.}  Let $\HH$ be a separable Hilbert 
space with the inner product $\iprod{\cdot}{\cdot}$ and  norm $\norm{\cdot} = \sqrt{\iprod{\cdot}{\cdot}}$. We let 
$v\otimes v$ be the outer product of vector $v\in\HH$, which acts on $u\in\HH$ as the rank-one operator $\pa{v\otimes v}u = \iprod{u}{v}v$. 
For a linear operator $A$ acting on $\HH$,  we  let $A^*$ be  its adjoint and  $\norm{A} = 
\sqrt{\trace{A^*A}}$ the Frobenius norm. An operator $A$ is positive semi-definite (PSD) if it satisfies $\iprod{x}{Ax}\ge 0$ for all $x\in 
\HH$ and 
Hermitian if $A = A^*$.   We use $A\succcurlyeq 0$ to denote that an operator $A$ is Hermitian and PSD (in short, HPSD). For HPSD operators 
$A$ and $B$, we use $A\succcurlyeq B$ to denote $A - B \succcurlyeq 0$. For a HPSD operator $A$, we use $A^{1/2}$ to denote the unique
HPSD operator satisfying $A^{1/2}A^{1/2} = A$. The identity operator on $\HH$ is denoted as $I$. Besides standard asymptotic notation 
like $\OO\pa{\cdot}$ or $o\pa{\cdot}$, we will sometimes use the cleaner but less standard notation $a\lesssim b$ to denote $a = 
\OO\pa{b}$. We will consider algorithms that interact with stochastic data in an sequential fashion. The sequence of random variables 
observed during the interaction induce a filtration $\pa{\F_t}_{t\ge 0}$. We will use the shorthand $\EEt{\cdot} = \EEcc{\cdot}{\F_t}$ to 
denote expectations conditional on the history. 

\section{Preliminaries}\label{sec:pre}
We study the problem of linear regression under the square loss, more commonly known as linear  least squares regression. The 
objective in this problem is to minimize the expected risk
\begin{equation}\label{LS}
 R(w) = \EE{\pa{\iprod{w}{x} - y}^2},
\end{equation}
where $w\in\HH$ is a parameter vector, $x\in\HH$ is a covariate and $y\in\real$ is a label, with $(x,y)$ drawn from a fixed (but unknown) 
distribution $\DD$. Letting 
$
 \Sigma = \EE{x\otimes x}
$
denote the covariance operator, the minimizer of the risk is given by
\begin{equation}\label{inverse}
 w^* = \Sigma^{-1} \EE{x y}
\end{equation}
and satisfies $R(w^*) = \inf_{w\in \HH} R(w)$. We assume $w^*$ to exist, since in general this might not be true in infinite dimensions. 
Also we abuse the notation in Eq.~\eqref{inverse} since in general $\Sigma$ is not invertible and a pseudoinverse should be considered. 
This choice is  made only to ease the notation.

We study algorithms that take as  input a set of data points $\ev{(x_t,y_t)}_{t=1}^n$ drawn identically and independently from $\DD$ and 
output a weight vector $w$  to approximately  minimize~\eqref{LS}. The quality of an approximate solution is measured by the 
the excess risk
\[
 \Delta(w) = R(w) - R(w_*) = \norm{\Sigma^{1/2} \pa{w - w^*}}^2.
\]
To compute a solution from data,  we consider the stochastic gradient method, a.k.a.~SGD, that for least squares takes the form
\begin{align}\label{eq:SGD}
 w_{t} &= w_{t-1} - \eta_t \pa{x_{t} \iprod{x_{t}}{w_{t-1}} - x_{t} y_{t}},
\end{align}
where $(\eta_t)_t>0$ is a sequence of stepsizes (or learning rates), and $w_0 \in \HH$ is an  initial point.
Typically, a decaying stepsize sequence is  chosen to ensure convergence, see  \cite{NJLS09} and references therein.
A result relevant to our 
study is obtained  by \citet{BM13} for a finite dimensional setting ($\cal H$ is of dimension $d$). 
Unlike previous results, here it is shown that convergence for a constant stepsize $\eta$ can be proved if  Polyak--Ruppert (PR) 
averaging 
\[
 \bw_n = \frac{1}{n+1} \sum_{t=0}^n w_t
\]
is considered.   This result was later strengthened in 
various respects by \citet{DB15} and \citet*{DFB16},  notably by weakening the assumptions in \citet{BM13} and separating error terms  
related to ``bias'' and ``variance''.  We highlight one result from  \citet{DFB16} which considers the sequence
\begin{equation}\label{noise_iter}
w_{t} = w_{t-1} - \eta (\Sigma w_{t-1} - x_{t} y_{t} +\lambda w_{t-1}) ,
\end{equation}
and  under  technical assumptions discussed later,  proves the excess-risk bound 
\begin{equation}\label{eq:reg_bound}
 \EE{\Delta(\bw_n)} \lesssim \frac{\sigma^2 \trace{\Sigma^2 \pa{\Sigma+\lambda I}^{-2}}}{n} + \pa{\lambda + \frac{1}{\eta n}}^2 
\norm{\Sigma^{1/2} \pa{\Sigma + \lambda I}^{-1}\pa{w_0 - w^*}},
\end{equation}
where $\sigma^2 > 0$ is an upper bound on the variance of the label noise.
The iteration~\eqref{noise_iter} is only of theoretical interest since the covariance $\Sigma$ is not known in practice. However, the 
obtained bound is simpler to present allows easier comparison with our result. A bound slightly more complex   than~\eqref{eq:reg_bound} 
can be obtained when $\Sigma$ is not known  \citep[Theorem~2]{DFB16}.

In this paper, we propose a generalized version of Polyak--Ruppert averaging that we call \emph{geometric Polyak--Ruppert averaging}. 
Specifically, the algorithm we study computes the standard SGD iterates as given by Equation~\eqref{eq:SGD} and outputs
\begin{equation}\label{eq:geo}
 \tw_n = \frac{1}{\sum_{k=0}^n \pa{1-\eta\lambda}^k} \cdot \sum_{t=0}^n \pa{1-\eta\lambda}^t w_t
\end{equation}
after round $n$, where $\lambda \in [0,1/\eta)$ is a tuning parameter. That is, the output is a geometrically discounted (and appropriately 
normalized) average of the plain SGD iterates that puts a higher weight on initial iterates. It is easy to see that setting $\lambda = 0$ 
exactly recovers the standard form of Polyak--Ruppert averaging. 
Our main result essentially shows that the resulting estimate $\tw_n$ satisfies\footnote{The bound shown here concerns an iteration similar 
to the one 
shown on Equation~\eqref{noise_iter}, and is proved in Appendix~\ref{app:additive}. We refer to Theorem~\ref{thm:main} for the precise 
statement of our main result.}
\[
 \EE{\Delta(\tw_n)} \lesssim \pa{\frac{\eta \lambda}{2} + \frac{1}{n}}  {\sigma^2 \trace{\Sigma^2 \pa{\Sigma+\lambda I}^{-2}}} + 
\pa{\lambda + \frac{1}{\eta n}}^2 
\norm{\Sigma^{1/2} \pa{\Sigma + \lambda I}^{-1}\pa{w_0 - w^*}}
\]
under the same assumptions as the ones made by \citet{DFB16}. Notably, this guarantee matches the bound of Equation~\eqref{eq:reg_bound}, 
with the key difference being that the factor $\frac 1n$ in the first term is replaced by  $\frac 1n + \frac{\eta\lambda}{2}$. 
This observation suggests the (perhaps surprising) conclusion that geometric Polyak--Ruppert averaging has a regularization effect 
qualitatively similar to Tikhonov regularization. Before providing the proof of the main result stated above, we first show that this 
similarity is more than a coincidence. Specifically, we begin by showing in Section~\ref{sec:expectation} that the limit of the weighted 
iterates is \emph{exactly} the ridge regression solution on expectation.

\section{Geometric iterate averaging realizes Tikhonov regularization on expectation}\label{sec:expectation}
We begin by studying the relation between the averaged iterates of unregularized SGD and the iterates of regularized SGD on expectation. 
This setting will allow us to make minimal assumptions: We merely assume that $\E{\norm{x}^2}<\infty$ so that the covariance 
operator $\Sigma$ satisfies $\Sigma \preccurlyeq B^2I$ for some $B>0$.
For ease of exposition, we assume in this section that $w_0 = 0$. First, we notice the relation 
\begin{align*}
 \EEt{w_t} &= w_{t-1} - \eta \pa{\EEt{x_{t} \otimes x_{t}}w_{t-1} - \EEt{x_{t} y_{t}}} = \pa{I-\eta \Sigma} w_{t-1} + \eta \EE{xy}
\end{align*}
between $w_t$ and $w_{t-1}$, which can be iteratively applied to obtain
\[
 \EE{w_t} = \eta \sum_{k=0}^t \pa{I - \eta \Sigma}^k \EE{x y}.
\]
In contrast we also define the iterates of regularized SGD with stepsize $\gamma > 0$ and regularization parameter $\lambda$ as
\begin{equation}\label{eq:sgdtik}
 \hw_{t} = \hw_{t-1} - \gamma \pa{x_{t} \iprod{x_{t}}{w_{t-1}} - x_{t} y_{t} + \lambda \hw_{t-1}}, 
\end{equation}
which can be similarly shown to satisfy
\[
 \EE{\hw_t} = \gamma \sum_{k=0}^t \pa{I - \gamma \Sigma - \gamma \lambda I}^k \EE{xy}.
\]
This latter definition can be seen as an empirical version of the iteration in Eq.~\eqref{noise_iter}.

The following proposition reveals a profound connection between the limits of $\EE{\hw_t}$ and the geometrically discounted average 
$\EE{\tw_t}$ as $t\ra \infty$ (given that the stepsizes are small enough for the limits to exist). 
\begin{proposition}\label{prop:limit_eq}
 Let $\eta$, $\gamma$ and $\lambda$ be such that $\gamma \lambda < 1$ and $\eta = \frac{\gamma}{1-\gamma\lambda}$, and assume that $\gamma 
\le \frac{1}{B^2}$. Then, $\hw_\infty = \lim_{t\ra\infty} \EE{\hw_t}$ and $\tw_\infty = \lim_{t\ra\infty} \EE{\tw_t}$ both exist 
and satisfy
 \[\hw_\infty = \frac{\gamma}{\eta} \tw_\infty = \pa{\Sigma + \lambda I}^{-1}\EE{xy}.\]
\end{proposition}
\begin{proof}
The analysis crucially relies on defining the geometric random variable $G$ with law $\PP{G=k} = \gamma \lambda \pa{1-\gamma \lambda}^k$ 
for 
all $k=1,2,\dots$ and noticing that $\PP{G\ge k} = \pa{1-\gamma\lambda}^k$. We let $\EEG{\cdot}$ stand for taking expectations with respect 
to the distribution of $G$. 
 First we notice that the limit of $\EE{w_t}$ can be written as
 \[
  \lim_{t\ra\infty} \EE{\hw_t} = \gamma \sum_{t=0}^\infty \pa{I - \gamma \Sigma - \gamma \lambda I}^t \EE{xy}.
 \]
 By our assumption on $\gamma$, we have $\gamma \Sigma \preccurlyeq I$, which implies that then the limit on the right-hand side exists 
and satisfies $\sum_{t=0}^\infty \pa{I - \gamma \Sigma - \gamma \lambda I}^t = \pa{\gamma \Sigma + \gamma \lambda I}^{-1}$. Having 
established the existence of the limit, we rewrite the regularized SGD iterates~\eqref{eq:sgdtik} as
\begin{align*}
  \lim_{t\ra\infty} \EE{\hw_t} &= \gamma \sum_{k=0}^\infty \pa{I - \gamma \Sigma - \gamma \lambda I}^k \EE{xy}
  = \gamma \sum_{k=0}^\infty \pa{I - \eta \Sigma}^k (1-\gamma\lambda)^k \EE{xy}
  \\
  &= \gamma \sum_{k=0}^\infty \pa{I - \eta \Sigma}^k \PP{G\ge k} \EE{xy} 
  = \gamma \EEG{\sum_{k=0}^\infty \pa{I - \eta \Sigma}^k \II{G\ge k} \EE{xy}}
  \\
  &= \gamma \EEG{\sum_{k=0}^G \pa{I - \eta \Sigma}^k  \EE{xy}} = \gamma  \EEG{\frac 1\eta \EE{w_G}}
  \\
  &= \frac{\gamma}{\eta} \sum_{k=0}^\infty \PP{G = k} \EE{w_k} = \frac{\gamma}{\eta} \sum_{k=0}^\infty \gamma \lambda (1-\gamma 
\lambda)^k \EE{w_k}
  \\
  &= \frac{\gamma}{\eta} \lim_{t\ra\infty} \sum_{k=0}^t \gamma \lambda (1-\gamma \lambda)^k \EE{w_k} = \frac{\gamma}{\eta} 
\lim_{t\ra\infty} \pa{1-\pa{1-\gamma\lambda}^t} \EE{\tw_t}
  \\
  &= \frac{\gamma}{\eta} \lim_{t\ra\infty} \EE{\tw_t}.
\end{align*}
This concludes the proof.
\end{proof}
\noindent It is useful to recall that $(\Sigma +\lambda I)^{-1}\EE{xy}$ is the solution of the problem
\[
\min_{w\in\HH }\EE{(y-\iprod{w}{x})^2}+\lambda \norm{w}^2,
\]
that is Tikhonov regularization applied to the expected risk or, in other words, the population version of the ridge regression estimator. 
Then, the above result shows that SGD with geometric average~\eqref{eq:geo} and the regularized iteration~\eqref{eq:sgdtik} both converge 
to this same solution.
Besides this result, we can also show that the expected iterates $\EE{\hw_t}$ and $\EE{\tw_t}$ are also closely connected for finite 
$t$, without any assumption on the learning rates.
\begin{proposition}\label{prop:finite_eq}
 Let $\eta$, $\gamma$ and $\lambda$ be such that $\gamma \lambda < 1$ and $\eta = \frac{\gamma}{1-\gamma\lambda}$. Then, 
 \[\EE{\hw_t} = \frac{\gamma}{\eta}\pa{1 - \pa{1-\gamma \lambda}^t} \cdot \EE{\tw_t}.\]
\end{proposition}
This proposition is proved using the same ideas as Proposition~\ref{prop:limit_eq}; we include the proof in Appendix~\ref{app:exp} for 
completeness.

\section{Main result: Finite-time performance guarantees}\label{sec:main} 
While the previous section establishes a strong connection between the geometrically weighted SGD iterates with the iterates of regularized 
SGD on expectation, this connection is clearly not enough for the known performance guarantees to carry over to our algorithm. 
Specifically, 
the 
two iterative schemes propagate noise differently, thus the covariance of the resulting iterate sequences may be very different from each 
other. In this section, we prove our main result that shows that the prediction error of our algorithm also behaves similarly to that of 
SGD with Tikhonov regularization. For our analysis, we will make the same assumptions as \citet{DFB16} and we will borrow several 
ideas from them, as well as most of their notation.

We now state the assumptions that we require for proving our main result. We first state an assumption on the fourth moment of the 
covariates.
\begin{assumption}\label{ass:cov}
 There exists $R>0$ such that
 \[
  \EE{\norm{x}^2 x \otimes x} \preccurlyeq R^2\Sigma.
 \]
\end{assumption}
This assumption implies that $\trace{\Sigma} \le R^2$, and is satisfied, for example, when the covariates satisfy 
$\norm{x}\le R$ almost surely. We always assume a minimizer $w_*$ of the expected risk  to exist and  also make an assumption on the 
residual $\varepsilon$ defined as the random variable
\[
 \varepsilon = y - \iprod{w^*}{x}.
\]
It is easy to show that $\EE{\varepsilon x} = 0$, even though $\EEcc{\varepsilon}{x}=0$ does not hold in general. We make the following 
assumption on the residual:
\begin{assumption}\label{ass:res}
 There exists $\sigma^2>0$ such that
 \[
  \EE{\norm{\varepsilon}^2 x \otimes x} \preccurlyeq \sigma^2 \Sigma.
 \]
\end{assumption}
This assumption is satisfied when $\norm{x}$ and $y$ are almost surely bounded, or when the model is well-specified and corrupted with 
bounded noise (i.e., when $\varepsilon$ is independent of $x$ and has variance bounded by $\sigma^2$).
Under the above assumptions, we prove the following bound on the prediction error---our main result:
\begin{theorem}\label{thm:main}
Suppose that Assumptions~\ref{ass:cov} and~\ref{ass:res} hold and  assume that $\eta\le \frac{1}{2R^2}$ and $\lambda \in [0,1/\eta)$. Then, 
the iterates computed by the recursion given by Equations~\eqref{eq:SGD} 
and~\eqref{eq:geo} satisfy
\begin{align*}
\EE{\Delta(\tw_n)}
 \le & \frac{4}{1-\gamma\lambda} \pa{\frac{\gamma\lambda}{\pa{2-\gamma\lambda}} +
\frac{2}{\pa{2-\gamma\lambda}\pa{n+1}}} \frac{\sigma^2\trace{\Sigma^2 \pa{\Sigma + 
\lambda I}^{-2}}}{2 - \eta R^2}
\\
&+2\pa{\lambda + \frac{1}{\gamma\pa{n+1}}}^2  \norm{ \Sigma^{1/2} \pa{\Sigma + \textstyle{\frac \lambda 2}I}^{-1} \pa{w_0 - w^*}}^2
\\
&  + \pa{\lambda + \frac{1}{\gamma\pa{n+1}}}^2 \trace{\Sigma \pa{\Sigma + \lambda I}^{-1}}\norm{\pa{\Sigma + \textstyle{\frac{\lambda}{2}}
I}^{-1/2}\pa{w_0 - w^*}}^2 
\end{align*} 
\end{theorem}
The (rather technical) proof of the theorem closely follows the proof of Theorem~2 of \citet{DFB16}. We describe the main components of the 
proof of our main result in Section~\ref{sec:analysis}.  For didactic purposes, we also present a simplified version of our analysis where 
we assume full knowledge of $\Sigma$ in Appendix~\ref{app:additive}.
\subsection{Discussion}\label{sec:discussion}
We  next  discuss various aspects and implications of our results. 
\paragraph{Comparison with \citet{DFB16}.} Apart from constant factors\footnote{By enforcing $\gamma\lambda \le 1/2$, the $1-\gamma\lambda$ 
and $2-\gamma\lambda$ terms in the denominator can be lower bounded by a constant.}, our bound above precisely matches that of Theorem~1 of 
\citet{DFB16}, except for an additional term of order $\gamma\lambda \sigma^2 \trace{\Sigma^2 \pa{\Sigma}^{-2}}$. This term, however, is 
\emph{not} a mere artifact of our proof: in fact it captures a distinctive noise-propagating effect of geometric PR averaging scheme. 
Indeed,  the regularization effect of our iterate averaging scheme is different from that of Tikhonov-regularized SGD in one significant 
way: while Tikhonov regularization increases the 
bias and strictly decreases the variance, our scheme may actually increase the variance for certain choices of $\lambda$. To see this, 
observe that setting a large $\lambda$ puts a large weight on the initial iterates, so that the initial noise is amplified compared to 
noise 
in the later stages, and the concentration of the total noise becomes worse. We note however that this extra term does not qualify 
as a serious limitation, since the commonly recommended setting $\lambda = \OO\bpa{\frac{1}{\eta n}}$ still preserves the optimal rates 
for both the bias and the variance up to constant factors.

\paragraph{Optimal excess risk bounds.} 
The bound in Theorem~\ref{thm:main} is essentially the same as the one derived in  \citet{DFB16}[Theorem 2].
Following their same reasoning,   the  bound  can be optimized with respect to $\lambda, \gamma$ to derive the best parameters choice and 
explicit upper bounds on the corresponding excess risk. In the finite dimensional case,  it is easy  to derive a bound of order ${\cal O} 
(d/n)$, which is known to be optimal in a minmax sense \citep{T08}. In the infinite dimensional case,  optimal  minmax bound  can again  be 
easily derived,  and also refined  under further  assumptions on $w_*$ and the covariance $\Sigma$ \citep{DCR05,CD07}. We omit this 
derivation. 

%%%%%%%%%%%%%%%%%%%%%%%%%%%
\paragraph{When should we set $\lambda > 0$?} We have two answers depending on the dimensionality of the underlying Hilbert space $\HH$. 
For infinite dimensional spaces, it is clearly necessary to set $\lambda > 0$.
In the finite-dimensional case, the advantage of our regularization scheme is less clear at first sight: while Tikhonov regularization 
strictly decreases the variance, this is not necessarily true for our scheme (as discussed above). A closer look reveals that, 
under some (rather interpretable) conditions, we can reduce the variance, as quantified by the following proposition.
\begin{proposition}\label{prop:varred}
 If $\trace{\Sigma^{-1}} > \frac 12 \gamma d n$ there exists a regularization parameter $\lambda^*>0$ satisfying
 \[
  \pa{\frac{\gamma \lambda^*}{2} + \frac{1}{n}}\trace{\Sigma^2 \pa{\Sigma + \lambda^*I}^{-2}} < \frac{d}{n}.
 \]
\end{proposition}
\begin{proof}
 Letting $s_1,s_2,\dots,s_d$ be the eigenvalues of $\Sigma$ sorted in decreasing order, we have
 \begin{align*}
  f(\lambda) = \pa{\frac{\gamma \lambda^*}{2} + \frac{1}{n}}\trace{\Sigma^2 \pa{\Sigma + \lambda^*I}^{-2}} &= 
  \pa{\frac{\gamma \lambda^*}{2} + \frac{1}{n}} \sum_{i=1}^d \frac{s_i^2}{\pa{s_i+\lambda}^2}.
 \end{align*}
Taking derivative of $f$ with respect to $\lambda$ gives
\[
  f'(\lambda) = \frac{\gamma}{2} \sum_{i=1}^d \frac{s_i^2}{\pa{s_i+\lambda}^2} - 2 \pa{\frac{\gamma \lambda^*}{2} + \frac{1}{n}} 
\sum_{i=1}^d 
\frac{s_i^2}{\pa{s_i+\lambda}^3}.
\]
In particular, we have
\[
 f'(0) = \frac{\gamma d}{2} - \frac{2}{n}\sum_{i=1}^d \frac{1}{s_i} = \frac{\gamma d}{2} - \frac{2\trace{\Sigma^{-1}}}{n},
\]
so $f'(0)>0$ holds whenever $\trace{\Sigma^{-1}} > \frac 14 \gamma d n$.
The proof is concluded by observing that $f'(0)>0$ implies the existence of a $\lambda^*$ with the claimed property.
\end{proof}
Intuitively, Proposition~\ref{prop:varred} suggests that the geometric PR averaging can definitely reduce the variance over standard 
PR averaging whenever the covariance matrix is poorly conditioned and/or the sample size is small. Notice however that 
the above argument only shows one example of a good choice of $\lambda$; many other good choices may exist, but these are harder to 
characterize.
\paragraph{Geometric averaging vs.~tail averaging.}
It is interesting to contrast our approach with the \emph{tail averaging} scheme studied by \citet{JKKNS16,JKKNPS17}: instead of putting 
large weight on the initial iterates as our method does, \citeauthor{JKKNPS17} suggest to average the \emph{last} $n-\tau$ iterates of SGD 
for some $\tau$. 
The effect of this operation is that the $\norm{\Sigma^{-1/2}\pa{w^*-w_0}}^2n^{-2}$ term arising from Polyak--Ruppert averaging is 
replaced by a term of order $\exp(-\eta \mu \tau)\norm{w^*-w_0}^2$, where $\mu>0$ is the smallest eigenvalue of $\Sigma$. Clearly, this 
yields a significant asymptotic speedup, but gives no advantage when $\gamma n \le \mu^{-1}$ (noting that $\tau<n$). Contrasting this 
condition with our Proposition~\ref{prop:varred} leads to an interesting conclusion: for small values of $n$, geometric averaging has an 
edge over tail averaging and vice versa.

\paragraph{What is the computational advantage?}
The main practical  advantage of our averaging scheme over Tikhonov regularization is a computational one: validating regularization 
parameters  becomes trivially easy to parallelize. Indeed, one can perform a \emph{single} pass of unregularized SGD over the data, store 
the iterates 
and average them with various schedules to evaluate different choices of $\lambda$. Through parallelization, this 
approach can achieve huge computational speedups over running regularized SGD from scratch. To see this, observe that the averaging 
operation is \emph{entirely non-serial}: one can cut the (stored) SGD iterates into $K$ contiguous batches and let each individual worker 
perform a geometrically discounted averaging with the same discount factor $(1-\gamma \lambda)$. The resulting averages are then combined 
by 
the master with appropriate weights. In contrast, regularized SGD is entirely serial, so validation cannot be parallelized.

\paragraph{Choosing the right averaging.} Finally, we note that the same method as above can be used to choose the correct parameter $\tau$ 
for tail averaging. This suggests a simple and highly parallelizable scheme for choosing the right averaging, with the identity of 
the best scheme depending on the interaction between $n$ and $\Sigma$ as discussed above.

\paragraph{Connections to early stopping.} A close inspection of the proofs of our Propositions~\ref{prop:limit_eq} 
and~\ref{prop:finite_eq} reveals an interesting perspective on geometric iterate averaging: Thinking of the averaging operation as 
computing a probabilistic expectation, one can interpret the geometric average as a \emph{probabilistic early stopping} method where the 
stopping time is geometrically distributed. Early stopping is a very well-studied regularization method for multipass stochastic gradient 
learning algorithms that has been observed and formally proved to have effects similar to Tikhonov regularization \citep{YRC07, RV15}. So 
far, all\footnote{With the notable exception of \citet{Fle90}, who shows an exact relation between the ridge-regression solution 
and a rather complicated early stopping rule involving a preconditioning step that requires computing the eigendecomposition of $\Sigma$.} 
published results that we are aware of merely point out the qualitative similarities between the performance bounds obtained for these two 
regularization methods, showing that using the stopping time should be chosen as $t^* = 1/\gamma \lambda$. In contrast, our 
Propositions~\ref{prop:limit_eq} and~\ref{prop:finite_eq} show a much deeper connection: geometric random stopping with expected stopping 
time $\EE{G} = 1/\gamma\lambda$ not only recovers the performance bounds, but \emph{exactly} recovers the ridge-regression solution.

\paragraph{Open questions.} It is natural to ask whether geometric iterate averaging has similar regularization effects in other stochastic 
approximation settings too. An possible direction for future work is studying the effects of our averaging scheme on accelerated and/or 
variance-reduced variants of SGD.
 Another promising direction is studying general linear stochastic approximation 
schemes \citep{LS18}, and particularly Temporal-Difference learning algorithms for Reinforcement Learning that have so far resisted all 
attempts to regularize them \citep{SB98,S10,F11}.

\section{The proof of Theorem~\ref{thm:main}}\label{sec:analysis}
Our proof closely follows that of \citet[Theorem~1]{DFB16}, with the key differences that 
\begin{itemize}
 \item we do not have to deal with an explicit ``regularization-based'' error term that gives rise to a term proportional to $\lambda^2$ 
in their bound, and
 \item the $\frac{1}{n}$ factors for iterate averaging are replaced by $c_n (1-\gamma\lambda)^t$ for each round, where 
 \[
c_n = \frac{1}{\sum_{t=1}^n (1-\gamma\lambda)^t} = \frac{\gamma\lambda}{1 - \pa{1-\gamma\lambda}^n}.
 \]
 As we will see, this change will propagate through the analysis and will eventually replace the $\frac {1}{\gamma^2 n}$
 and $\pa{2\lambda + \frac{1}{\eta n}}$ factors in the final bound by $c_n^2 \sum_{t=1}^n \pa{1-\gamma\lambda}^{2t}$ and 
$\frac{c_n^2}{\gamma^2}$, respectively.
\end{itemize}
In the interest of space, we only provide an outline of the proof here and defer the proofs of the key lemmas to Appendix~\ref{app:proofs}. 
Throughout the proof, we will suppose that the conditions of Theorem~\ref{thm:main} hold.
The lemma below shows that the factors involving $c_n^2$ are of the order claimed in the Theorem.
\begin{lemma}\label{lem:cnbound}
 For any $n\ge 1$, we have
 \[
  c_n^2 \le \pa{\gamma\lambda + \frac{1}{n+1}}^2
 \]
 and 
 \[
  c_n^2 \sum_{t=1}^n \pa{1-\gamma\lambda}^{2t} \le \frac{\gamma\lambda}{\pa{2-\gamma\lambda}} +
\frac{2}{\pa{2-\gamma\lambda}\pa{n+1}}.
 \]
\end{lemma}
The straightforward proof is given in Appendix~\ref{app:cnbound}.

Now we are ready to lay out the proof of Theorem~\ref{thm:main}.
Let us start by introducing the notation
\[
 M_{i,j} = \pa{\prod_{k=i+1}^j \pa{I - \eta x_k \otimes x_k}}.
\]
and recalling the definition $\varepsilon_t = y_t - \iprod{x_t}{w^*}$. 
A simple recursive argument shows that
\begin{equation}\label{eq:w_decomp}
 \begin{split}
 w_t - w^* &= w_{t-1} - \eta \pa{ \iprod{x_t}{w_{t-1}} - y_t} x_t - w^*
 \\
 &= \pa{I - \eta x_t \otimes x_t} \pa{w_{t-1}- w^*} + \eta x_t y_t - \eta x_t \iprod{x_t}{w^*}
 \\
 &= \pa{I - \eta x_t \otimes x_t} \pa{w_{t-1}- w^*} + \eta \varepsilon_t x_t 
 \\
 &= M_{t-1,t} \pa{w_{t-1}- w^*} + \eta \varepsilon_t x_t 
 \\
 &= M_{0,t} \pa{w_0 - w^*} + \eta \sum_{j=1}^t M_{j,t} \varepsilon_j x_j.
 \end{split}
\end{equation}
Thus, the averaged iterates satisfy
\begin{align*}
 \tw_n - w^* &= c_n \cdot \sum_{t=0}^{n} \pa{1-\gamma\lambda}^t \pa{w_t - w^*}
 \\
 &= c_n \cdot \sum_{t=0}^{n} \pa{1-\gamma\lambda}^t \pa{M_{0,t} \pa{w_0 - w^*} + \eta \sum_{j=1}^t M_{j,t} \varepsilon_j x_j}.
\end{align*}
We first show a simple upper bound on the excess risk $\Delta(\tw_n) = \norm{\Sigma^{1/2} \pa{\tw_t - w^*}}^2$:
\begin{lemma}\label{lem:riskbound1}
 \[
  \EE{\Delta(\tw_n)} \le 
  \frac{2 c_n^2}{\gamma} \sum_{t=0}^n (1-\gamma\lambda)^{2t} \EE{\norm{\Sigma^{1/2}\pa{\Sigma + \lambda I}^{1/2} \pa{w_t - w^*}}^2}
 \]
\end{lemma}
The proof is included in Appendix~\ref{app:riskbound1}.
In order to further upper bound the right-hand side in the bound stated in Lemma~\ref{lem:riskbound1}, we can combine the decomposition of 
$w_t - w^*$ in Equation~\eqref{eq:w_decomp} with the Minkowski inequality to get
\begin{equation}\label{eq:risk_decomp}
 \begin{split}
 &\sum_{t=0}^n (1-\gamma\lambda)^{2t} \EE{\norm{\Sigma^{1/2}\pa{\Sigma + \lambda I}^{1/2} \pa{w_t - w^*}}^2}
 \\
 & \le 2 \underbrace{\sum_{t=0}^n (1-\gamma\lambda)^{2t} \EE{\norm{\Sigma^{1/2}\pa{\Sigma + \lambda I}^{1/2} M_{0,t} \pa{w_0 - 
w^*}}^2}}_{\Delta_{1}}
 \\
 &\qquad\qquad\qquad\qquad\qquad\qquad + 2 \eta^2 \sum_{t=0}^n (1-\gamma\lambda)^{2t} \underbrace{\EE{\norm{\Sigma^{1/2}\pa{\Sigma + 
\lambda I}^{1/2}\sum_{j=1}^t M_{j,t} \varepsilon_j x_j}^2}}_{\Delta_{2,t}}
 \end{split}
\end{equation}
The first term in the above decomposition can be thought of as the excess risk of a ``noiseless'' process (where $\sigma = 0$) and the 
second term as that of a ``pure noise'' process (where $w_0 = w^*$). The rest of the analysis is devoted to bounding these two terms.

We begin with the conceptually simpler case of bounding $\Delta_2,t$, which can be done uniformly for 
all $t$. In particular, we have the following lemma:
\begin{lemma}\label{lem:noise}
For any $t$, we have
 \[
  \Delta_{2,t} \le \frac{\eta \sigma^2}{2 - \eta R^2} \trace{\Sigma^2 \pa{\Sigma + \lambda I}^{-2}}.
 \]
\end{lemma}
The rather technical proof is presented in Appendix~\ref{app:noise}.
% \redd{THIS IS A FUCKING MESS BUT IT WORKS.} 
We now turn to bounding the excess risk of the ``noiseless'' process, $\Delta_1$:
\[
 \Delta_1
 =
 \sum_{t=0}^n (1-\gamma\lambda)^{2t} \cdot \EE{\trace{M_{0,t}^* \Sigma \pa{\Sigma + \lambda I}^{-1} M_{0,t} \pa{w_0 - w^*}\otimes \pa{w_0 - 
w^*}}}.
\]
The following lemma states a bound on $\Delta_1$.
\begin{lemma}\label{lem:noiseless}
 \[
  \Delta_1 \le \frac {1}{2\gamma} \trace{\Sigma \pa{\Sigma + \lambda I}^{-1} \pa{\Sigma + \textstyle{\frac{\lambda}{2}} I}^{-1} E_0} + 
\frac {1}{4\gamma} \trace{\Sigma \pa{\Sigma + \lambda I}^{-1}}\trace{\pa{\Sigma + \textstyle{\frac{\lambda}{2}} I}^{-1}E_0} 
 \]
\end{lemma}
The extremely technical proof of this theorem is presented in Appendix~\ref{app:noiseless}.

The proof of Theorem~\ref{thm:main} is concluded by plugging the bounds of Lemmas~\ref{lem:noise} and~\ref{lem:noiseless} 
into Equation~\eqref{eq:risk_decomp} and using Lemma~\ref{lem:riskbound1} to obtain
\begin{align*}
 \EE{\Delta(\tw_n)}
 \le &\frac{c_n^2}{\gamma^2} \pa{2 \trace{\Sigma \pa{\Sigma + \lambda I}^{-1} \pa{\Sigma + \textstyle{\frac{\lambda}{2}} I}^{-1} E_0} 
+\trace{\Sigma 
\pa{\Sigma + \lambda I}^{-1}}\trace{\pa{\Sigma + \textstyle{\frac{\lambda}{2}} I}^{-1}E_0}}
\\
&+ \frac{4 \eta c_n^2}{\gamma} \sum_{t=0}^n (1-\gamma\lambda)^{2t} \frac{\sigma^2 \trace{\Sigma^2 \pa{\Sigma + 
\lambda I}^{-2}}}{2 - \eta R^2} 
\end{align*}
Now we can finish by using the bounds on $c_n^2$ and $c_n^2 \sum_{t=0}^n (1-\gamma\lambda)^{2t}$ given in 
Lemma~\ref{lem:cnbound}.

\acks
GN is supported by the UPFellows Fellowship (Marie Curie COFUND program n${^\circ}$ 600387).
LR is funded by the Air Force project FA9550-17-1-0390 (European Office of Aerospace Research and Development) and by the FIRB project 
RBFR12M3AC (Italian Ministry of Education, University and Research). The authors thank Francesco Orabona, Csaba Szepesv\'ari and G\'abor 
Lugosi for interesting discussions.

\bibliography{sgd}

\appendix

\section{The proof of Proposition~\ref{prop:finite_eq}}\label{app:exp}
The proof is similar to that of Proposition~\ref{prop:limit_eq}, although a little more cluttered due to the normalization constants 
involved in the definition of $\tw_t$. The analysis in this case relies on defining the \emph{truncated} geometric random variable $G$ 
with law 
\[
\PP{G = k} = \frac{\gamma \lambda}{1-(1-\gamma\lambda)^t} (1-\gamma\lambda)^{k}
\]
for all $k=0,1,2,\dots,t$. Using the notation $b_t = 1-(1-\gamma\lambda)^t$, we have $\PP{G\ge k} = b_t^{-1} (1-\gamma \lambda)^k$. 
Again, letting $\EEG{\cdot}$ stand for taking expectations with respect to the distribution of $G$, we rewrite the regularized SGD 
iterates as
\begin{align*}
 \EE{\hw_t} &= \gamma \sum_{k=0}^t \pa{I - \gamma \Sigma - \gamma \lambda I}^k \EE{xy}
 = \gamma \sum_{k=0}^t \bpa{\pa{1-\gamma \lambda}I - \gamma \Sigma}^k \EE{xy}
 \\
 &= \gamma \sum_{k=0}^t \pa{I - \frac{\gamma}{1-\gamma \lambda} \Sigma}^k \pa{1-\gamma \lambda}^k \EE{xy}
 = \gamma \sum_{k=0}^t \pa{I - \eta \Sigma}^k (1-\gamma\lambda)^k \EE{xy}
 \\
 &= \gamma \sum_{k=0}^t \pa{I - \eta \Sigma}^k \bpa{b_t \PP{G \ge k}} \EE{xy} = \gamma b_t \EEG{\sum_{k=0}^t \pa{I - \eta \Sigma}^k \II{G 
\ge k} \EE{xy}}
 \\
 &= \gamma b_t \EEG{\sum_{k=0}^G \pa{I - \eta \Sigma}^k \EE{xy}} = \gamma b_t \EEG{\frac{1}{\eta}\EE{w_G}} 
 \\
 &= \frac{\gamma}{\eta} \sum_{k=0}^{t} b_t \PP{G = k} \EE{w_k}
= \frac{\gamma}{\eta} \sum_{k=0}^{t} \gamma \lambda (1-\gamma\lambda)^{k} \EE{w_k}
 \\
 &= \frac{\gamma}{\eta} \pa{1 - \pa{1-\gamma\lambda}^t} \EE{\tw_t},
\end{align*}
where we used  $\sum_{k=0}^t (1-\gamma\lambda)^k = \frac{1 - \pa{1-\gamma\lambda}^t}{\gamma \lambda}$ and the definition of $\tw_t$ in the 
last step. This 
concludes the proof.

\section{Tools for proving Theorem~\ref{thm:main}}\label{app:proofs}

\subsection{The proof of Lemma~\ref{lem:cnbound}}\label{app:cnbound}
Regarding the first statement, we have
\begin{align*}
 c_n &= \frac{1}{\sum_{t=0}^n (1-\gamma\lambda)^t} = \frac{\gamma\lambda}{1 - (1-\gamma\lambda)^{n+1}}
 \le \frac{\gamma\lambda}{1 - e^{-\gamma\lambda\pa{n+1}}} \le \gamma\lambda \pa{1 + \frac{1}{\gamma\lambda\pa{n+1}}},
\end{align*}
where the first inequality uses $1-x\le e^{-x}$ that holds for all $x\in\real$ and the second one uses $\frac{1}{1-e^{-x}} \le \frac 1x + 
1$ that holds for all $x>0$.
The second statement is proven as
\begin{align*}
 c_n^2 \sum_{t=1}^n \pa{1-\gamma\lambda}^{2t} &= \frac{\sum_{t=1}^n \pa{1-\gamma\lambda}^{2t} }{\pa{\sum_{t=0}^n 
(1-\gamma\lambda)^t}^2} = \frac{\pa{\gamma\lambda}^2 }{\pa{1-\pa{1-\gamma\lambda}^2}}\cdot\frac{\pa{1 - 
(1-\gamma\lambda)^{2\pa{n+1}}}}{\pa{1 - 
(1-\gamma\lambda)^{n+1}}^2}
\\
&=
\frac{\pa{\gamma\lambda}^2 }{\gamma\lambda\pa{2-\gamma\lambda}} \cdot \frac{1 + (1-\gamma\lambda)^{\pa{n+1}}}{1 - (1-\gamma\lambda)^{n+1}} 
\le
\frac{\gamma\lambda}{2-\gamma\lambda} \cdot \frac{1 + e^{-\gamma\lambda\pa{n+1}}}{1 - e^{-\gamma\lambda \pa{n+1}}}
\\
&=
\frac{\gamma\lambda}{2-\gamma\lambda} \cdot \pa{1 + \frac{2 e^{-\gamma\lambda\pa{n+1}}}{1 - e^{-\gamma\lambda \pa{n+1}}}}
\le
\frac{\gamma\lambda}{2-\gamma\lambda} \cdot \pa{1 + \frac{2}{\gamma\lambda \pa{n+1}}},
\end{align*}
where the first inequality again uses $1-x\le e^{-x}$ that holds for all $x\in\real$ and the second one uses $\frac{e^{-x}}{1-e^{-x}} \le 
\frac 1x$ that holds for all $x>0$.
\jmlrQED

\subsection{The proof of Lemma~\ref{lem:riskbound1}}\label{app:riskbound1}
We start by noticing that
\begin{align}\label{eq:expansion}
 &\norm{\Sigma^{1/2} \pa{\bw_n - w^*}}^2 = c_n^2 \sum_{t=0}^n \sum_{k=0}^n (1-\gamma\lambda)^{t+k} \Iprod{w_t - w^*}{\Sigma 
\pa{w_k - w^*}}
\\
&\quad= c_n^2 \sum_{t=0}^n (1-\gamma\lambda)^{2t} \Iprod{w_t - w^*}{\Sigma \pa{w_t - w^*}}  + 2 c_n^2 \sum_{t=0}^n 
\sum_{k=t+1}^n (1-\gamma\lambda)^{t+k} \Iprod{w_t - w^*}{\Sigma \pa{w_k - w^*}}.\nonumber
\end{align}
To handle the second term, we first notice that for any $t$ and $k>t$, we have
\begin{align*}
 \EEt{\Iprod{w_t - w^*}{\Sigma \pa{w_k - w^*}}} &= \EEt{\iprod{w_t - w^*}{\Sigma \pa{M_{t,k} \pa{w_t - w^*} + \eta 
\sum_{j=i+1}^k M_{j,k} \varepsilon_j x_j}}}
\\
&=  \iprod{w_t - w^*}{\Sigma \pa{I - \gamma \Sigma}^{k-t} \pa{w_t - w^*}},
\end{align*}
where we used $\EEt{\varepsilon_k x_k} = 0$ that holds for $k>t$ and $\EEt{M_{t,k}} = 
\pa{I - \gamma \Sigma}^{k-t}$. Using this insight, we obtain
\begin{align*}
&c_n^2 \EE{\sum_{t=0}^n \sum_{k=t+1}^n (1-\gamma\lambda)^{t+k} \Iprod{w_t - w^*}{\Sigma \pa{w_k - w^*}}}
\\
 &\quad = c_n^2 \EE{\sum_{t=0}^n \sum_{k=t+1}^n (1-\gamma\lambda)^{t+k} \Iprod{w_t - w^*}{\Sigma \pa{I - \gamma \Sigma}^{k-t} \pa{w_t - 
w^*}}}
 \\
 &\quad = c_n^2 \EE{\sum_{t=0}^n (1-\gamma\lambda)^t \iprod{w_t - w^*}{\Sigma \pa{\sum_{k=t+1}^{n} \pa{1-\gamma\lambda}^k 
\pa{I-\eta\Sigma}^{k-t} \pa{w_t - w^*}}}}
\\
 &\quad = c_n^2 \EE{\sum_{t=0}^n (1-\gamma\lambda)^t \iprod{w_t - w^*}{\Sigma \pa{\sum_{k=t+1}^{n} \pa{1-\gamma\lambda}^t 
\pa{I-\eta\pa{1-\gamma\lambda}\Sigma - \gamma\lambda I}^{k-t} \pa{w_t - w^*}}}}
\\
 &\quad \le c_n^2 \EE{\sum_{t=0}^n (1-\gamma\lambda)^{2t} \iprod{w_t - w^*}{\Sigma \pa{\sum_{k=t+1}^{\infty} \pa{I-\gamma\Sigma - 
\gamma\lambda I}^{k-t} \pa{w_t - w^*}}} }
\\
 &\qquad\qquad\qquad\mbox{(adding positive terms to the sum)}
\\
 &\quad = \frac{c_n^2}{\gamma} \EE{\sum_{t=0}^n (1-\gamma\lambda)^{2t} \iprod{w_t - w^*}{\Sigma \pa{\Sigma + \lambda 
I}^{-1} \pa{I-\gamma\Sigma - \gamma\lambda I} \pa{w_t - w^*}}}
\\
 &\qquad\qquad\qquad\mbox{(using $\sum_{j=1}^\infty (I - A)^j = A^{-1} \pa{I - A}$ that holds for $A\preccurlyeq I$)}
\\
 &\quad \le \frac{c_n^2}{\gamma} \EE{\sum_{t=0}^n (1-\gamma\lambda)^{2t} \iprod{w_t - w^*}{\Sigma \pa{\Sigma + \lambda I}^{-1} 
\pa{w_t - w^*}}}  
\\
&\qquad\qquad\qquad\qquad\qquad\qquad\qquad\qquad\qquad - c_n^2 \EE{\sum_{t=0}^n (1-\gamma\lambda)^{2t} \iprod{w_t - w^*}{\Sigma \pa{w_t - 
w^*}}}.
\end{align*}
Noticing that the last term matches the first term on the right-hand side of Equation~\eqref{eq:expansion}, the proof is concluded.
\jmlrQED

\subsection{The proof of Lemma~\ref{lem:noise}}\label{app:noise}
The proof of this lemma crucially relies on the following inequality:
\begin{lemma}\label{lem:recurrence}
 Assume that $\eta \le \frac{1}{2R^2}$. Then, for all $k<t$, we have
 \[
  \EE{M_{k+1,t} \Sigma M_{k+1,t}^*} \preccurlyeq \frac{1}{\eta \pa{2 - \eta R^2}} \pa{\EE{M_{k+1,t} 
\Sigma\pa{\Sigma + \lambda I}^{-1}M_{k+1,t}^*} - \EE{M_{k,t} \Sigma\pa{\Sigma + \lambda I}^{-1}M_{k,t}^*}}.
 \]
\end{lemma}
\begin{proof}
 We have
 \begin{align*}
  &\EE{M_{k,t} \Sigma \pa{\Sigma + \lambda I}^{-1} M_{k,t}^*} 
  \\
  &= \EE{M_{k+1,t} \pa{I - \eta x_{k+1}\otimes x_{k+1}}\Sigma \pa{\Sigma + 
\lambda I}^{-1} \pa{I - \eta x_{k+1} \otimes x_{k+1}} M_{k+1,t}^*}
\\
&= \EE{M_{k+1,t} \Sigma\pa{\Sigma + \lambda I}^{-1}\pa{I - 2\eta x_{k+1} \otimes x_{k+1}}M_{k+1,t}^*} 
\\
&\qquad + \eta^2\EE{M_{k+1,t} \pa{x_{k+1} \otimes x_{k+1}} \Sigma \pa{\Sigma + \lambda I}^{-1} \pa{x_{k+1} \otimes x_{k+1}} 
M_{k+1,t}^*}
\\
&\preccurlyeq \EE{M_{k+1,t} \pa{\Sigma\pa{\Sigma + \lambda I}^{-1}  - 2\eta \Sigma^2 \pa{\Sigma + \lambda I}^{-1}  + 
\eta^2 R^2 \Sigma^2 \pa{\Sigma + \lambda I}^{-1}} M_{k+1,t}^*}
\\
&= \EE{M_{k+1,t} \Sigma\pa{\Sigma + \lambda I}^{-1}M_{k+1,t}^*} - \eta \pa{2 - \eta R^2} \EE{ M_{k+1,t} \Sigma^2 \pa{\Sigma + 
\lambda I}^{-1} M_{k+1,t}}
 \end{align*}
 Reordering, we obtain
 \begin{align*}
  &\EE{ M_{k+1,t} \Sigma^2 \pa{\Sigma + \lambda I}^{-1} M_{k+1,t}} 
  \\
  &\qquad\qquad \le \frac{1}{\eta \pa{2 - \eta R^2}} \pa{\EE{M_{k+1,t} 
\Sigma\pa{\Sigma + \lambda I}^{-1}M_{k+1,t}^*} - \EE{M_{k,t} \Sigma\pa{\Sigma + \lambda I}^{-1}M_{k,t}^*}}.
 \end{align*}
 The result follows from noticing that $\Sigma \preccurlyeq \Sigma^2 \pa{\Sigma + \lambda I}^{-1}$.
\end{proof}
Now we are ready to prove Lemma~\ref{lem:noise}.

For any $t$, we have
\begin{align*}
&\Delta_{2,t} = \EE{\sum_{k=1}^t \sum_{j=1}^t 
\iprod{\varepsilon_j M_{j,t} x_j}{\Sigma \pa{\Sigma + \lambda I}^{-1} M_{k,t} x_k \varepsilon_k}}
\\
&= \EE{\sum_{k=1}^t \iprod{\varepsilon_k M_{k,t}x_k}{\Sigma  \pa{\Sigma + \lambda I}^{-1} M_{k,t} x_k 
\varepsilon_k}} 
\\
&\qquad\qquad\qquad+ 2 \EE{\sum_{k=1}^t \sum_{j=k+1}^{t} \iprod{\varepsilon_j M_{j,t}x_j}{\Sigma \pa{\Sigma + 
\lambda I}^{-1} M_{k,t} x_k \varepsilon_k}}
\\
&= \trace{\EE{\sum_{k=1}^t \varepsilon_k^2 M_{k,t} \pa{x_k \otimes x_k} M_{k,t}^* \Sigma \pa{\Sigma + \lambda I}^{-1}}}
\qquad(\mbox{using that $\EEk{\varepsilon_j x_j} = 0$ for $j> k$})
\\
&\le \sigma^2 \trace{\sum_{k=1}^t M_{k,t} \Sigma M_{k,t}^* \Sigma \pa{\Sigma + \lambda I}^{-1}},
\end{align*}
where the last inequality uses our assumption on the noise that $\EE{\varepsilon_k^2 \pa{x_k \otimes x_k}} \preccurlyeq \sigma^2 \Sigma$.
Now, using Lemma~\ref{lem:recurrence}, we obtain
\begin{align*}
 &\trace{\sum_{k=1}^t \EE{M_{k,t} \Sigma M_{k,t}}^* \Sigma \pa{\Sigma + \lambda I}^{-1}}
 \\
 &\le \frac{1}{\eta \pa{2 - \eta R^2}} \trace{\sum_{k=1}^t \pa{\EE{M_{k,t} 
\Sigma\pa{\Sigma + \lambda I}^{-1}M_{k,t}^*} - \EE{M_{k-1,t} \Sigma\pa{\Sigma + \lambda I}^{-1}M_{k-1,t}^*}} \Sigma 
\pa{\Sigma + \lambda I}^{-1}}
 \\
 &\le \frac{1}{\eta \pa{2 - \eta R^2}} \trace{\pa{\EE{M_{t-1,t} 
\Sigma\pa{\Sigma + \lambda I}^{-1}M_{t-1,t}^*} - \EE{M_{0,t} \Sigma\pa{\Sigma + \lambda I}^{-1}M_{0,t}^*}} \Sigma 
\pa{\Sigma + \lambda I}^{-1}}
 \\
 &\le \frac{1}{\eta \pa{2 - \eta R^2}} \trace{\Sigma^2 \pa{\Sigma + \lambda I}^{-2}},
\end{align*}
where the last step uses that $\EE{M_{t-1,t}} = \Sigma$. This concludes the proof.
\jmlrQED

\subsection{The proof of Lemma~\ref{lem:noiseless}}\label{app:noiseless}
We begin by noticing that 
\[
 \Delta_1
 \le
 \sum_{t=0}^\infty (1-\gamma\lambda)^{2t} \cdot \EE{\trace{M_{0,t}^* \Sigma \pa{\Sigma + \lambda I}^{-1} M_{0,t} \pa{w_0 - w^*}\otimes 
\pa{w_0 - w^*}}}
\]
holds since the sum only has positive elements.
Following \citet{DFB16} again, we define the operator $\TT$ acting on an arbitrary Hermitian operator $A$ as
\[
 \TT A = \Sigma A + A \Sigma - \eta \EE{\iprod{x_t}{A x_t} x_t\otimes x_t},
\]
and also introduce the operator $\Sw$ defined as
\[
 \Sw A = \EE{\iprod{x_t}{A x_t} x_t\otimes x_t},
\]
so that $\TT A = \Sigma A + A \Sigma - \eta \Sw A$. We note that $\Sw$ and $\TT$ are Hermitian and positive definite (the latter being true 
by our assumption about $\eta$). Finally, we define $\Iw$ as the identity operator acting on Hermitian matrices. With this notation, we can 
write
\[
 \EE{M_{0,t}^*A M_{0,t}} = \pa{\Iw - \eta \TT}^t A.
\]
Thus, defining $E_0 = \pa{w_0 - w^*}\otimes\pa{w_0 - w^*}$, we have
\begin{align*}
 &\sum_{t=0}^\infty (1-\gamma\lambda)^{2t} \EE{\trace{M_{0,t}^* \Sigma \pa{\Sigma + \lambda I}^{-1} M_{0,t} E_0}} 
 \\
 &\qquad \qquad =
 \sum_{t=0}^\infty (1-\gamma\lambda)^{2t} \trace{\pa{\Iw - \eta \TT}^{t} \left[\Sigma \pa{\Sigma + \lambda I}^{-1}\right] E_0}
 \\
 &\qquad \qquad =
 \sum_{t=0}^\infty (1-\gamma\lambda)^{t} \trace{\pa{\Iw - \eta\pa{1-\gamma\lambda} \TT - \gamma\lambda \Iw}^{t} \left[\Sigma 
\pa{\Sigma + \lambda I}^{-1}\right] E_0}
\\
 &\qquad \qquad \le
 \sum_{t=0}^\infty \trace{\pa{\Iw - \gamma \TT - \gamma\lambda \Iw}^{i} \left[\Sigma \pa{\Sigma + \lambda I}^{-1}\right] E_0}
 \\
 &\qquad \qquad =
 \frac{1}{\gamma}\trace{\pa{\TT + \lambda \Iw}^{-1}\left[\Sigma \pa{\Sigma + \lambda I}^{-1}\right] E_0},
\end{align*}
where the last step holds true if $\norm{\Iw - \eta \TT} < 1$. For a proof of this fact, we refer to Lemma~5 in \citet{DB15}.
Let us define $\TT_\lambda = \TT + \lambda \Iw$ and $W = \TT_\lambda^{-1}\left[\Sigma \pa{\Sigma + \lambda I}^{-1}\right]$, so that it 
remains to bound $ \gamma^{-1} \trace{WE_0}$. We notice that, by definition, $W$ satisfies
\begin{align}\label{eq:sigma-W}
 \Sigma \pa{\Sigma + \lambda I}^{-1} = \Sigma W + W \Sigma + \lambda W - \eta \Sw W.
\end{align}
Also introducing the operators $\UU_L$ and $\UU_R$ as the left- and right-multiplication operators with $\Sigma$, respectively, we get 
after reordering that
\begin{align*}
 W &= \pa{\UU_L + \UU_R + \lambda I}^{-1} \Sigma \pa{\Sigma + \lambda I}^{-1} + \eta \pa{\UU_L + \UU_R + \lambda I}^{-1} \Sw W
 \\
 &= \frac 12 \Sigma \pa{\Sigma + \lambda I}^{-1} \pa{\Sigma + \textstyle{\frac{\lambda}{2}} I}^{-1} + \eta \pa{\UU_L + \UU_R + \lambda 
I}^{-1} 
\Sw W.
\end{align*}
Using the fact that $\UU_L + \UU_R + \lambda I$ and its inverse are Hermitian, we can show
\begin{align*}
 \trace{WE_0} = \frac 12 \trace{\Sigma \pa{\Sigma + \lambda I}^{-1} \pa{\Sigma + \textstyle{\frac{\lambda}{2}} I}^{-1} E_0} + \eta 
\trace{\Sw W \pa{\UU_L + \UU_R + \lambda I}^{-1}E_0}.
\end{align*}
Furthermore, by again following the arguments\footnote{
This result  is  proven for finite dimension by \citet{DFB16}, but  can be easily generalized to infinite dimensions.}  of 
\citet[pp.~28]{DFB16}, we can also show
\[
 \pa{\UU_L + \UU_R + \lambda I}^{-1}E_0 \preccurlyeq \frac{\iprod{w_0 - w^*}{\pa{\Sigma + \frac{\lambda}{2} I}^{-1} \pa{w_0 - 
w^*}}}{2} I.
\]
Since $\Sw W$ is positive, this leads to the bound
\begin{align}\label{eq:webound}
 \trace{WE_0} \le \frac 12 \trace{\Sigma \pa{\Sigma + \lambda I}^{-1} \pa{\Sigma + \textstyle{\frac{\lambda}{2}} I}^{-1} E_0} + 
\frac{\eta\iprod{w_0 - 
w^*}{\pa{\Sigma + \frac{\lambda}{2} I}^{-1} \pa{w_0 - 
w^*}}}{2} \trace{\Sw W},
\end{align}
so it remains to bound $\trace{\Sw W}$. On this front, we have
\begin{align*}
 \trace{\Sw W} &= \trace{\EE{\iprod{x_t}{W x_t} x_t\otimes x_t}} \le R^2 \trace{W \Sigma}
\end{align*}
by our assumption on the covariates. Also, by Equation~\eqref{eq:sigma-W}, we have
\begin{align*}
 \trace{\Sigma \pa{\Sigma + \lambda I}^{-1}} &= 2 \trace{W\Sigma} + \lambda \trace{W} - \eta \trace{\Sw W}
 \\
 &= 2 \trace{W\pa{\Sigma + \textstyle{\frac{\lambda}{2}} I}} - \eta \trace{\Sw W}
 \\
 &\ge 2 \trace{W \Sigma } - \eta \trace{\Sw W}
 \\
 &\ge \pa{\frac{2}{R^2} - \eta} \trace{\Sw W}
 \\
 &\ge \frac{ \trace{\Sw W}}{R^2},
\end{align*}
by crucially using the assumption $\eta \le 1/2R^2$. Plugging into Equation~\eqref{eq:webound} and using $\eta R^2 \le 2$ again proves the 
lemma.
\jmlrQED

\section{Analysis under additive noise}\label{app:additive}
In this section, we consider the ``additive-noise'' model of \citet{DFB16}.  This model assumes that the learner has 
access to the gradient estimator $f_t'(w) = \Sigma w - y_t x_t$, that is, the gradient is subject to the noise vector $\xi_t = y_t x_t - 
\EE{y_t x_t}$ which doesn't depend on the parameter vector $w$. This assumption is admittedly very strong, and we mainly include our 
analysis for this case for didactic purposes. Indeed, the analysis for this case is significantly simpler than in the setting considered in 
the main body of our paper.

In this setting, stochastic gradient descent takes the form
\begin{equation}\label{eq:additive}
w_{t} = w_{t-1} - \eta (\Sigma w_{t-1} - x_{t} y_{t}),
\end{equation}
which is the unregularized counterpart of the iteration already introduced in Section~\ref{sec:pre} as Equation~\eqref{noise_iter}.
Again, we will study the geometric average
\begin{equation}\label{eq:avg_additive}
 \tw_t = \frac{\sum_{k=0}^{t} \pa{1-\eta\lambda}^k w_k}{\sum_{j=0}^t (1-\eta\lambda)^j}.
\end{equation}
For the analysis, we recall the definition $\xi_t = x_t y_t - \EE{x_t y_t}$ and study the evolution of $\tw_t - w^*$:
\begin{align*}
 \tw_t - w^* &= \pa{I - \eta\Sigma} \pa{\tw_{t-1} - w^*} + \eta \xi_t
 \\
 &= \pa{I - \eta\Sigma}^t \pa{\tw_{0} - w^*} + \eta \sum_{k=1}^t \pa{I - \eta\Sigma}^{t-k} \xi_k.
\end{align*}
We prove the following performance guarantee about this algorithm:
\begin{proposition}
 Suppose that $V = \EE{\xi_t\otimes \xi_t} \le \tau^2 \Sigma$ for some $\sigma^2>0$ and assume that $\eta\le \lambda_{\max}(\Sigma)$ and 
$\lambda \in [0,1/\eta)$. Then, the iterates computed by the recursion given by Equations~\eqref{eq:additive} 
and~\eqref{eq:avg_additive} satisfy
 \begin{align*}
 \EE{\Delta(\tw_n)}
 &\le \pa{\frac{\eta}{\gamma}}^2\pa{\frac{\gamma\lambda}{\pa{2-\gamma\lambda}} + 
\frac{2}{ \pa{2-\gamma\lambda}\pa{n+1}}}\trace{\Sigma \pa{\Sigma + \lambda I}^{-2} V}
 \\
 &\qquad\qquad + \pa{\lambda + \frac{1}{\gamma\pa{n+1}}}^2 \norm{\Sigma^{1/2}\pa{\Sigma + \lambda I}^{-1} \pa{\tw_0 -
w^*}^2}.
\end{align*}
\end{proposition}
\begin{proof}
Recalling the notation $c_n = \pa{\sum_{t=0}^n (1-\gamma\lambda)^t}^{-1}$, the geometric average can be written as
\begin{align*}
 \tw_n - w^* &= c_n \cdot \sum_{t=0}^{n} \pa{1-\gamma\lambda}^t \pa{\tw_t - w^*}
 \\
 &= c_n \cdot \sum_{t=0}^{n} \pa{1-\gamma\lambda}^t \pa{\pa{I-\eta\Sigma}^t \pa{\tw_0 - w^*} + \eta \sum_{j=1}^t \pa{I - 
\eta\Sigma}^{t-j} \xi_j}
\\
 &= c_n \cdot \pa{\sum_{t=0}^{n} \pa{1-\gamma\lambda}^t \pa{I-\eta\Sigma}^t \pa{\tw_0 - w^*} + \eta \sum_{t=0}^{n} 
\pa{1-\gamma\lambda}^t \sum_{j=1}^t \pa{I - \eta\Sigma}^{t-j} \xi_j}
\\
 &= c_n \cdot \pa{\sum_{t=0}^{n} \pa{I-\gamma \Sigma - \gamma\lambda I}^t \pa{\tw_0 - w^*} + \eta \sum_{j=1}^n 
\pa{1-\gamma\lambda}^j \sum_{t=j}^{n} \pa{1-\gamma\lambda}^{t-j} \pa{I - \eta\Sigma}^{t-j} \xi_j}
\\
 &= c_n \cdot \pa{\sum_{t=0}^{n} \pa{I-\gamma \Sigma - \gamma\lambda I}^t \pa{\tw_0 - w^*} + \eta \sum_{j=1}^n 
\pa{1-\gamma\lambda}^j \sum_{t=j}^{n} \pa{I - \gamma\Sigma - \gamma \lambda I}^{t-j} \xi_j} 
\\
 &= c_n \cdot \pa{\sum_{t=0}^{n} \pa{I-\gamma \Sigma - \gamma\lambda I}^t \pa{\tw_0 - w^*} + \eta \sum_{j=1}^n 
\pa{1-\gamma\lambda}^j \sum_{t=0}^{n-j} \pa{I - \gamma\Sigma - \gamma \lambda I}^{t} \xi_j}
\\
 &= c_n \cdot \BPA{\frac{1}{\gamma}\pa{I - \pa{I-\gamma \Sigma - \gamma\lambda I}^{-t}} \pa{\Sigma + \lambda I}^{-1} \pa{\tw_0 - 
w^*} \\
 &\qquad + \frac{\eta}{\gamma} \sum_{j=1}^n 
\pa{1-\gamma\lambda}^j \pa{I-\pa{I-\gamma\Sigma - \gamma\lambda I}^{t-j}} \pa{\Sigma + \lambda I}^{-1} \xi_j}
\end{align*}
Now, exploiting the assumption that the noise is i.i.d.~and zero-mean, we get
\begin{align*}
 \EE{\Delta(\tw_n)} &= \EE{\norm{\Sigma^{1/2}\pa{\tw_t - w^*}}^2} 
 \\
 &= c_t^2 \cdot \pa{\frac{\eta}{\gamma}}^2 \sum_{j=1}^t \pa{1-\gamma\lambda}^{2j} 
 \trace{\pa{I-\pa{I-\gamma\Sigma - \gamma\lambda I}^{t-j}}^2 \Sigma \pa{\Sigma + \lambda I}^{-2} V}
 \\
 &\qquad + c_t^2 \cdot \frac{1}{\gamma^2}\norm{\pa{I - \pa{I-\gamma \Sigma - \gamma\lambda I}^{-t}} \Sigma^{1/2}\pa{\Sigma + 
\lambda}^{-1} 
\pa{\tw_0 - w^*}}^2
 \\
 &\le c_t^2 \pa{\frac{\eta}{\gamma}}^2 \sum_{j=1}^t \pa{1-\gamma\lambda}^{2j} \trace{\Sigma \pa{\Sigma + \lambda I}^{-2} V}
 \\
 &\qquad + c_t^2 \cdot \pa{\frac{1}{\gamma}}^2 \norm{\Sigma^{1/2}\pa{\Sigma + \lambda I}^{-1} \pa{\tw_0 - w^*}}^2,
\end{align*}
The proof is concluded by appealing to Lemma~\ref{lem:cnbound}.
\end{proof}

\end{document}